\newif\iftecrep\tecreptrue    
\newif\ifjmlr\jmlrfalse

\documentclass[11pt,twoside]{article}
\usepackage{natbib}
\bibpunct{(}{)}{;}{a}{,}{,}
\usepackage{mdwlist}
\usepackage{longtable}
\usepackage{latexsym,amsmath,amssymb}
\usepackage{wrapfig}
\usepackage{tikz}
\usepackage{dsfont}
\usepackage{appendix}
\usepackage[ruled]{algorithm}
\usepackage[noend]{algpseudocode}

\usetikzlibrary{trees,arrows,automata}

\newcommand{\sap}{s\!a^{\!+}}
\newcommand{\sam}{s\!a^{\!-}}
\newcommand{\piucrl}{\pi}
\newcommand{\aO}{\tilde O }
\newcommand{\tM}{{\widetilde M}}
\newcommand{\tV}{{\widetilde V}}
\newcommand{\tp}{{\tilde p}}
\newcommand{\tw}{\tilde w}
\newcommand{\ts}{{\tilde \sigma}}
\newcommand{\hp}{{\hat p}}
\newcommand{\hV}{{\widehat V}}
\newcommand{\hM}{{\widehat M}}

\newcommand{\Mh}{{M_{\operatorname{hard}}}}

\newcommand{\ki}{{\kappa, \iota}}
\newcommand{\KI}{{\K\times\I}}
\newcommand{\SA}{{S \times A}}

\newcommand{\ceil}[1]{\left \lceil {#1} \right\rceil}
\newcommand{\s}[1]{\;#1\;}
\renewcommand{\P}[1]{P\left\{#1\right\}}
\newcommand{\argmax}{\operatornamewithlimits{arg\,max}}
\newcommand{\argmin}{\operatornamewithlimits{arg\,min}}
\newcommand{\ind}[1]{[\![ #1 ]\!]}

\newcommand{\twonorm}[1]{\left\|#1\right\|_2}
\newcommand{\Var}{\operatorname{Var}}
\renewcommand{\O}[1]{O\left(#1 \right)}

\newcommand{\E}{\mathbf E}
\newcommand{\M}{{\mathcal M}}
\newcommand{\K}{{\mathcal K}}
\newcommand{\R}[0]{\mathbb R}
\newcommand{\N}[0]{\mathbb N}
\newcommand{\I}{{\mathcal I}}

\newcommand{\constm}{{20L_1|\KI||\mathcal D|^2 \over \epsilon^2 (1 - \gamma)^{2 + 2/\beta}}}
\newcommand{\constH}{{{1 \over 1 - \gamma} \log{8|S| \over \epsilon(1 - \gamma)}}}
\newcommand{\constiotamax}{{\ceil{{1 \over \log 2} \log {8|S| \over \epsilon(1 - \gamma)^2}}}}
\newcommand{\constN}{{6|\SA|m}}

\newcommand{\constdeltaone}{{\delta \over 2|\SA|\umax}}
\newcommand{\constwmin}{{\epsilon(1 - \gamma) \over 4|S|}}

\newcommand{\constbeta}{\ceil{{1 \over 2 \log 2} \log{1 \over 1 - \gamma}}}
\newcommand{\constexpmax}{4N|\KI|}
\newcommand{\constupdate}{|\SA|\log {|\SA| \over |\KI|}}
\newcommand{\constumax}{|\SA||\KI|}

\newcommand{\umax}{U_{\operatorname{max}}}
\newcommand{\expmax}{E_{\operatorname{{max}}}}
\newcommand{\wmin}{w_{\operatorname{{min}}}}
\newcommand{\iotamax}{{\iota_{\operatorname{max}}}}

\newcommand{\eqr}[1]{Equation (\ref{#1})}
\newcommand{\eqn}[1]{\begin{align}#1\end{align}}
\newcommand{\eq}[1]{\begin{align*}#1\end{align*}}

\usepackage{amsthm}
\theoremstyle{plain}
\newtheorem{theorem}{Theorem}
\newtheorem{lemma}[theorem]{Lemma}

\theoremstyle{definition}
\newtheorem{definition}[theorem]{Definition}
\newtheorem{assumption}[theorem]{Assumption}
\newtheorem{remark}[theorem]{Remark}

\theoremstyle{remark}

\renewcommand{\qedsymbol}{$\blacksquare$}

\newenvironment{proofof}[1]{\par\vspace{1mm}\noindent{\bfseries\upshape Proof#1.}}{\hfill\qedsymbol \par\vspace{2mm}\noindent\ignorespacesafterend}
\renewenvironment{proof}{\par\vspace{1mm}\noindent{\bfseries\upshape Proof.}}{\hfill\qedsymbol \par\vspace{2mm}\noindent\ignorespacesafterend}
\newenvironment{proofsketchof}[1]{\par\vspace{1mm}\noindent{\bfseries\upshape Proof sketch#1.}}{\hfill\qedsymbol \par\vspace{2mm}\noindent\ignorespacesafterend}
\newenvironment{keywords}{\centerline{\bf\small
Keywords}\begin{quote}\small}{\par\end{quote}\vskip 1ex}
\newcommand{\acks}[1]{\subsubsect{Acknowledgements} #1}
\def\subsubsect#1{\vspace{1ex plus 0.5ex minus 0.5ex}\noindent{\bf\boldmath{#1.}}}

\begin{document}

\title{
\vskip 2mm\bf\Large\hrule height5pt \vskip 4mm
PAC Bounds for Discounted MDPs
\vskip 4mm \hrule height2pt}
\author{{\bf Tor Lattimore}$^1$ and {\bf Marcus Hutter}$^{1,2,3}$ \\[3mm]
\normalsize Research School of Computer Science \\[-0.5ex] 
\normalsize $^1$Australian National University and $^2$ETH Z{\"u}rich and $^3$NICTA \\[-0.5ex]
\normalsize\texttt{\{tor.lattimore,marcus.hutter\}@anu.edu.au}
}
\date{January 2012}

\maketitle

\begin{abstract}
We study upper and lower bounds on the sample-complexity of learning near-optimal behaviour in finite-state
discounted Markov Decision Processes (MDPs).
For the upper bound we make the assumption that each action leads to at most two possible next-states and
prove a new bound for a UCRL-style algorithm on the number of time-steps when it is not Probably Approximately Correct (PAC).
The new lower bound strengthens previous work by being both more general (it applies to all policies) and tighter.
The upper and lower bounds match up to logarithmic factors.

\iftecrep\def\contentsname{\centering\normalsize Contents}
{\parskip=-2.7ex\tableofcontents}\fi
\end{abstract}


\newpage
\section{Introduction}

The goal of reinforcement learning is to construct algorithms that learn to act optimally, or nearly so, in unknown environments.
In this paper we restrict our attention to finite state discounted MDPs with unknown transitions.
The performance of reinforcement learning algorithms in this setting can be measured in a number of ways, for instance by using regret or
PAC
bounds \citep{Kak03}. We focus on the
latter, which is a measure of the number of time-steps where an algorithm is not near-optimal with high probability. Many previous algorithms
have been shown to be PAC with varying bounds \citep{Kak03,LS05,Str06,Str09,SS10,Aue11}.

We modify the Upper Confidence Reinforcement Learning (UCRL) algorithm of \citet{AJO10,Aue11,LS08} and, under the assumption
that there are at most two possible next-states for each state/action pair, prove a PAC bound of
\eq{
\aO\left({|\SA| \over \epsilon^2(1-\gamma)^3} \log{1 \over \delta}\right).
}
This bound is an improvement\footnote{In this slightly restricted setting.} on the previous best \citep{Aue11} and published best \citep{SS10}, which are
\eq{
\aO\left({|\SA| \over \epsilon^2(1-\gamma)^4} \log{1 \over \delta}\right)\qquad \text{and} \qquad
\aO\left({|\SA| \over \epsilon^2(1-\gamma)^6} \log{1 \over \delta}\right)
}
respectively.
The additional assumption is unfortunate and is probably unnecessary as discussed in Section \ref{sec_extension}.

We also present a matching (up to logarithmic factors) lower bound that is both larger and more general than the previous best given by
\citet{Str09}. The class of MDPs used in the counter-example satisfy the assumption used in the upper bound.

\section{Notation}

Unfortunately, we found it impossible to reduce the amount of notation and number of constants. While we have endeavoured to
define everything before we use it, readers are encouraged to consult the tables of notation and constants found in the appendix.

\subsubsect{General}
$\N = \left\{0,1,2,\cdots\right\}$ is the natural numbers. For the indicator function we write $\ind{x = y} = 1$ if $x = y$ and $0$ if $x \neq y$.
We use $\wedge$ and $\vee$ for logical and/or respectively. If $A$ is a set then $|A|$ is its size and $A^*$ is the set of all finite ordered subsets. Unless otherwise mentioned, $\log$ represents
the natural logarithm. For random variable $X$ we write $\E X$ and $\Var X$ for its expectation and variance respectively.
We make frequent use of the progression $z_i = 2^i - 2$ for $i \geq 1$. Define a set
$\mathcal Z(a) := \left\{z_i : 1 \leq i \leq \argmin_{i} \left\{z_i \geq a\right\} \right\}$.

\subsubsect{Markov Decision Process}
An MDP is a tuple $M = (S, A, p, r, \gamma)$ where $S$
and $A$ are finite sets of states and actions respectively. $r:S\to[0,1]$ is
the reward function. $p:\SA\times S\to[0,1]$ is the transition function and $\gamma \in (0, 1)$ the discount rate.
A stationary policy $\pi$ is a function $\pi:S \to A$ mapping a state to an action.
We write $p_{s,a}^{s'}$ as the probability of moving from state $s$ to $s'$ when taking action $a$
and $p_{s,\pi}^{s'} := p_{s,\pi(s)}^{s'}$.
The value of policy $\pi$ in $M$ and state $s$ is
$V^\pi_M(s) := r(s) + \gamma \sum_{s' \in S} p_{s,\pi}^{s'} V^\pi_M(s')$.
We view $V^\pi_M$ either as a function $V^\pi_M:S \to \R$ or a vector $V^\pi_M \in \R^{|S|}$ and similarly $p_{s,a} \in [0,1]^{|S|}$ is a vector.
The optimal policy of $M$ is defined $\pi^*_M := \argmax_{\pi} V^\pi_M$.
Common MDPs are $M$, $\hM$ and $\widetilde M$, which represent the true MDP, the estimated
MDP using empirical transition probabilities and a model.
We write $V := V_M$, $\widehat V := V_{\widehat M}$ and
$\widetilde V := V_{\widetilde M}$ for their values respectively. Similarly, $\hat \pi^* := \pi^*_{\hM}$ and
in general, variables with an MDP as a subscript will be written with a hat, tilde or nothing
as appropriate and the subscript omitted.

\section{Estimation}
In the next section we will introduce the new algorithm, but first we give an intuitive introduction to the type of parameter estimation
required to prove sample-complexity bounds for MDPs. The general idea is to use concentration inequalities to show the empiric estimate of a transition probability
approaches the true probability exponentially fast in the number of samples gathered. There are a wide variety of concentration inequalities, each catering
to a slightly different purpose. We improve on previous work by using Bernstein's inequality, which takes variance
into account (unlike Hoeffding).
The following example demonstrates the need for Bernstein's inequality when estimating the value functions of MDPs.
It also gives insight into the workings of the proof in the next two sections.

\setlength{\intextsep}{0pt}
\begin{wrapfigure}[6]{r}{3.4cm}
\topsep=0.0cm
\vspace{-0.1cm}
\small
\begin{tikzpicture}[->,>=stealth',shorten >=1pt,auto,node distance=2.0cm, semithick]
\tikzstyle{state} = [circle,draw,minimum width=0.8cm, minimum height=1.0cm]
\node[state] (plus) {$\stackrel{s_0} {r = 1}$};
\node[state] (minus) [right of=plus] {$\stackrel{s_1} {r = 0}$};
\path (plus) edge[bend left] node {$1 - p$} (minus)
      (plus) edge[loop below] node {$p$} (plus)
    (minus) edge[bend left] node {$1-q$} (plus)
    (minus) edge[loop below] node {$q$} (minus)
;
\end{tikzpicture}
\end{wrapfigure}
Consider the Markov reward process on the right with two states where rewards are shown inside the states and transition probabilities on the edges.
Note this is not an MDP because there are no actions.
We are only concerned with how well the value can be approximated.
Assume $p > \gamma$, $q$ arbitrarily large (but not $1$) and let $\hat p$ be the empiric estimate of $p$ and consider the error in our estimated value and the true value while in state $s_0$. One can show that
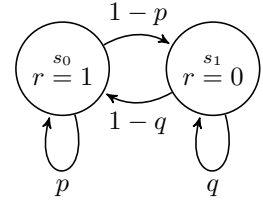
\eqn{\label{eqn-intuition}
\left|V(s_0) - \hV(s_0)\right| \s{\approx} {|\hat p - p| \over (1 - \gamma)^2}.
}
Therefore if $V - \hV$ is to be estimated to within $\epsilon$ accuracy, we need
$|\hat p - p| < \epsilon (1 - \gamma)^2$. Now suppose we bound $|\hat p - p|$ via a standard Hoeffding bound, then with high probability
$|\hat p - p| \lesssim \sqrt{L/n}$
where $n$ is the number of visits to state $s_0$ and $L = \log (1 / \delta)$.
Therefore to obtain an error less than $\epsilon(1 - \gamma)^2$ we need $n > {L \over \epsilon^2 (1 - \gamma)^4}$ visits to state $s_0$,
which is already too many for a bound in terms of $1/(1 - \gamma)^3$.
If Bernstein's inequality is used instead, then $|\hat p - p| \lesssim \sqrt{Lp(1-p) / n}$ and so
$n > {Lp(1-p) \over \epsilon^2 (1 - \gamma)^4}$ is required, but Equation (\ref{eqn-intuition}) depends on $p > \gamma$. Therefore
$n > {L \over \epsilon^2 (1 - \gamma)^3}$ visits are sufficient. If $p < \gamma$ then \eqr{eqn-intuition} can be improved.

\section{Upper Confidence Reinforcement Learning Algorithm}\label{sec_alg}
UCRL is based on the optimism principle for solving the exploration/exploitation dilemma. It is model-based
in the sense that at each time-step the algorithm acts according to a model (in this case an MDP, $\tM$) chosen from a model class. The idea
is to choose the smallest model class guaranteed to contain the true model with high probability and act according to the
most optimistic model within this class. With a good choice of model class this guarantees a policy that biases its exploration towards
unknown states that may yield good rewards while avoiding states that are known to be bad. The approach has been successful in obtaining
uniform sample complexity (or regret) bounds in various domains where the exploration/exploitation problem is an issue
\citep{LR85,Agr95,ACF02,LS05,AO07,AJO10,Aue11}.

Unfortunately, to prove our new bound we needed to make an assumption about the transition probabilities of the true MDP. We do not believe
this assumption is crucial, but it substantially eases the analysis by removing some dependencies in the more general problem. In
Section \ref{sec_extension} we present an approach to remove the assumption as well as some intuition into why this ought to be possible, but non-trivial.
\begin{assumption}\label{ass}
The true unknown MDP, $M$, satisfies $p_{s,a}^{s'} = 0$ for all but two $s' \in S$ denoted $\sap, \sam \in S$.\footnote{Note that $\sap$ and $\sam$
{\it are} dependent on $(s, a)$ and are known to the algorithm.}
\end{assumption}
The pseudo-code of UCRL can be found below, but first we define a {\it knownness} index, $\kappa$. If $n$ is the number of times a state/action pair
has been visited then $\kappa(\iota, n)$ is the knownness of that state/action pair at level $\iota$.
The knownness of a state increases with the number of visits, is bounded by $|S|$ and is always a natural number. The reason for defining
these now is that UCRL will only perform an update when the knownness index of some states would be changed by an update. Unfortunately, the
definition below is unlikely to be very intuitive. A more thorough explanation of knownness is given
in Section \ref{sec-upper}.
\begin{definition}[Knownness]
Define constants
\eq{
\wmin &:= \constwmin & w_\iota &:= 2^\iota \wmin & \iotamax &:= \constiotamax \\
\mathcal I &:= \left\{0,1,\cdots, \iotamax\right\} & \mathcal K &:= \mathcal Z(|S|).
}
We define the knownness index, $\kappa:\I \times \N \to \K$ by
\eq{
\kappa(\iota,n) & := \max \left\{ z \in \K : z \leq {n \over w_{\iota} m} \right\},
}
where $m \in \aO\left({{1 \over \epsilon^2(1 - \gamma)^2} \log{|\SA| \over \delta}}\right)$ is defined in Appendix \ref{app_const}.
\end{definition}
Note that the existence of the function \Call{ExtendedValueIteration}{} is proven and an algorithm given by \citet{LS08}.

\begin{algorithm}[H]
\caption{UCRL}
\small
\begin{algorithmic}[1]
\State $t = 1$, $k = 1$, $n(s, a) = n(s, a, s') = 0$ for all $s, a, s'$ and $s_1$ is the start state.
\State $H := \constH$, $L_1 := \log{2 \over \delta_1}$ and $\delta_1 := {\delta \over 2|\SA|^2 |\KI|}$
\Loop
\State $\hp_{s,a}^{\sap} := n(s, a) / \max\left\{1, n(s, a, \sap)\right\}$ and $\hp_{s,a}^{\sam} := 1 - \hp_{s,a}^{\sap}$
\State $\mathcal M_k := \left\{\tM : |\tp_{s,a}^{\sap} - \hp_{s,a}^{\sap}| \leq \Call{ConfidenceInterval}{\tp_{s,a}^{\sap}, n(s,a)},\; \forall (s,a) \right\}$
\State $\tM = \Call{ExtendedValueIteration}{\M_k}$
\State $\pi_k = \tilde \pi^*$
\State $v(s, a) = v(s, a, s') = 0$ for all $s, a, s'$
\While{$\kappa(\iota, n(s, a) + v(s, a)) = \kappa(\iota,n(s, a)) , \forall (s, a), \iota \in \I$}
\State \Call{Act}{}
\EndWhile
\State \Call{Delay}{} and \Call{Update}{}
\EndLoop
\Function{Delay}{}
\For{$j = 1 \to H$}
\State \Call{Act}{}
\EndFor
\EndFunction
\Function{Update}{}
\State $n(s, a) = n(s, a) + v(s, a)$ and $n(s, a, s') = n(s, a, s') + v(s, a, s')\; \forall s,a,s'$ and $k = k + 1$
\EndFunction
\Function{Act}{}
\State $a_t = \pi_k(s_t)$
\State $s_{t+1} \sim p_{s_t,a_t}$ \Comment{Sample from MDP}
\State $v(s_t, a_t) = v(s_t, a_t) + 1$ and $v(s_t, a_t, s_{t+1}) =v(s_t, a_t, s_{t+1}) + 1$ and $t = t + 1$
\EndFunction
\Function{ExtendedValueIteration}{$\mathcal M$}
\State {\bf return} optimistic $\tM \in \M$ such that $V_{\tM}^*(s) \geq V_{\tM'}^*(s)$ for all $s \in S$ and $\tM' \in \M$.
\EndFunction
\Function{ConfidenceInterval}{$p,n$}
\State \Return $\min \left\{\sqrt {{2L_1p(1 - p) \over n} } + {2L_1 \over 3n}, \; \sqrt{{L_1 \over 2n}}\right\}$
\EndFunction
\end{algorithmic}
\end{algorithm}
\section{Upper PAC Bounds}\label{sec-upper}

We present two new PAC bounds.
The first improves on all previous analysis, but relies on Assumption \ref{ass}.
The second is completely general, but
gains an additional dependence on $|S|$ leading to a PAC bound in terms of $|S|^2$ and $1/(1 - \gamma)^3$. This bound is worse than
the previous best in terms of $|S|$, but better in terms $1/(1 - \gamma)$.

\begin{theorem}\label{thm_upper}
Let $M$ be the true MDP satisfying Assumption \ref{ass}. Let $\pi$ be the actual (non-stationary) policy of UCRL (Algorithm 1),
then $V^*(s_t) - V^{\pi}(s_t) > \epsilon$ for at most
\eq{
H\umax + H\expmax \in
\O{{|\SA| \over \epsilon^2(1 - \gamma)^3} \log {|\SA| \over \delta\epsilon(1 - \gamma)} \log^2 |S|
\log^2 {|S| \over \epsilon(1 -\gamma)}\log^2 \log{1 \over 1 - \gamma}}
}
time-steps with probability at least $1 - \delta$. ($\umax$ and $\expmax$ are defined in Appendix \ref{app_const}.)
\end{theorem}
Note that although $\pi_k$ is stationary, the global policy of UCRL is non-stationary. Despite this, we will abuse notation
by allowing ourselves to write $V^\pi(s_t)$, whereas really $V^\pi$ should depend on the entire history.
Fortunately, when UCRL is not delaying, the policy $\pi$ is nearly stationary in the sense that it
will be so for the next $H$ time-steps. This allows us to work almost entirely with stationary policies and so discard
the cumbersome notation required for non-stationary policies.

\begin{theorem}\label{thm_upper2}
Let $M$ be the true MDP (possibly not satisfying Assumption \ref{ass}) then there exists a policy $\pi$ such that
$V^*(s_t) - V^\pi(s_t) > \epsilon$ for at most
$|S|\log^3|S| (\expmax H + \umax H)$ time-steps with probability at least $1 - \delta$.
\end{theorem}
The proof of Theorem \ref{thm_upper2} is omitted, but follows easily by converting an arbitrary MDP with $|S|$ states
into a functionally equivalent MDP with $O(|S|^2)$ states that satisfies Assumption \ref{ass}. This is done by adding a tree
of $2|S|$ states for each state/action pair and rescaling $\gamma$.

\subsubsect{Proof Overview}
The proof of Theorem \ref{thm_upper} borrows components from the work of \citet{AJO10}, \citet{LS08} and \citet{SS10}.
\begin{enumerate}
\item Bound the number of updates by $\constupdate$, which follows from the algorithm and the definition of knownness. This
bounds the number of delaying time-steps to $\aO({1 \over 1 -\gamma} \constupdate)$ time-steps, which is insignificant from the point of
view of Theorem \ref{thm_upper}.
\item Show that the true MDP remains in the model class $\mathcal M_k$ for all $k$.
\item Use the optimism principle to show that if $M \in \mathcal M_k$ and $V^* - V^\piucrl > \epsilon$ then $|\tV^{\pi_k} - V^{\pi_k}| > \epsilon/2$.
This key fact shows that if $\piucrl$ is not nearly-optimal at some time-step $t$ then the true value and model value of $\pi_k$ differ and
so some information is (probably) gained by following this policy.
\item The final component is to bound the number of time-steps when $\pi$ is not nearly-optimal.
\end{enumerate}

\subsubsect{Episodes and phases}
UCRL operates in {\it episodes}, which are blocks of time-steps ending when \Call{update}{} is called. The length of each episode is not fixed, instead, an
episode ends when the knownness of a state changes. We often refer to time-step $t$ and episode $k$ and unless there is ambiguity we will not define $k$ and
just assume it is the episode in which $t$ resides.
A {\it delay phase} is the period of $H$ contiguous time-steps where UCRL is in the function \Call{delay}{}, which happens immediately before an update.
An {\it exploration phase}
is a period of $H$ time-steps starting at $t$ where $t$ is not in a delay phase and where
$\tV^{\pi_k}(s_t) - V^{\pi_k}(s_t) \geq \epsilon / 2$. Exploration phases do note overlap. More formally, the starts of exploration phases, $t_1, t_2, \cdots$,
are defined inductively
\eq{
t_1 &:= \min \left\{t : \tV^{\pi_k}(s_t) - V^{\pi_k}(s_t) \geq \epsilon/2 \wedge t \text{ is not in a delay phase} \right\}  \\
t_i &:= \min \left\{t : t \geq t_{i-1} + H \wedge \tV^{\pi_k}(s_t) - V^{\pi_k}(s_t) \geq \epsilon/2 \wedge t \text{ is not in a delay phase}  \right\}.
}
Note there need not, and with high probability will not, be infinitely many such $t_i$. The exploration phases are only used in the
analysis, they are not known to UCRL.

\subsubsect{Weights and variances}
We define the weight\footnote{Also called the discounted future state distribution in \citet{Kak03}.} of state/action pair $(s, a)$ as follows.
\eq{
w^\pi(s, a | s') &\s{:=} \ind{(s', \pi(s')) = (s, a)} + \gamma \sum_{s''} p_{s',\pi(s')}^{s''} w^\pi(s,a|s'') & w_t(s) := w^{\pi_k}(s, \pi_k(s) | s_t).
}
As usual, $\tilde w$ and $\hat w$ are defined as above but with $p$ replaced by $\tilde p$ and $\hat p$ respectively.
Think of $w_t(s)$ as the expected number of discounted visits to state/action pair $(s, \pi_k(s))$ while following policy $\pi_k$ starting
in state $s_t$.
The important point is that this value is approximately equal to the expected number of visits to state/action pair
$(s, \pi_k(s))$ within the next $H$ time-steps.
We also define the local variance of the value function. These measure the variability of values while following policy $\pi$.
\eq{\sigma^\pi(s)^2 &:= p_{s,\pi} \cdot V^{\pi^2} - [p_{s,\pi} \cdot V^\pi]^2 &
\tilde \sigma^\pi(s)^2 &:= \tp_{s,\pi} \cdot \tV^{\pi^2} - [\tp_{s,\pi}\cdot \tV^\pi]^2.
}

\subsubsect{The active set}
We will shortly see that states with small $w_t(s)$ cannot influence the differences in value functions. Thus we define an {\it active}
set of states where $w_t(s)$ is not tiny.
At each time-step $t$ define the {\it active} set $X_t$ by
\eq{
X_t \s{:=} \left\{s : w_t(s) > \constwmin =: w_{min}\right\}.
}

\subsubsect{Knownness}
We now expand on the concept of knownness and explain its purpose.
We write $n_t(s, a)$ for the value of $n(s,a)$ at time-step $t$ and $n_t(s) := n_t(s, \pi_k(s))$ where
$k$ is the episode associated with time-step $t$.
Let $t$ be some non-delaying time-step and suppose $s$ is active ($s \in X_t$). Now let $\iota_t(s) := \argmin_{\iota} {w_t(s) > w_\iota}$
and note that $\iota_t(s) \in \mathcal I$. We define a partition of the active set $X_t$ by
\eq{
K_t(\ki) := \left\{s \in X_t: \iota_t(s) = \iota \wedge \kappa_t(\iota_t(s), n_t(s)) = \kappa \right\}.
}
The set $K_t(\ki)$ represents a set of states that have comparable weights and visit counts.
We will show that if $|K_t(\ki)| \leq \kappa$ for all $\ki$ then the values $\tV$ and $V$ are reasonably close.
This result forms a key stage in the proof of Theorem \ref{thm_upper} because it shows that if $\pi$ is not nearly-optimal at
time-step $t$ then there exists a
$K_t(\ki)$ that is quite large and where states have not been visited sufficiently. Furthermore, the weights $w_t(s)$ where
$s \in K_t(\ki)$ are large enough that some learning is expected to occur.

\subsubsect{Analysis}
The proof of Theorem \ref{thm_upper} follows easily from three key lemmas.

\begin{lemma}\label{lem_optimism}
The following hold:
\begin{enumerate*}
\item The total number of updates is bounded by $\umax := \constupdate$.
\item If $M \in \mathcal M_k$ and $t$ is not in a delay phase and
$V^*(s_t) - V^{\piucrl}(s_t) > \epsilon$ then
\eq{\tV^{\pi_k}(s_t) - V^{\pi_k}(s) > \epsilon/2.
}
\end{enumerate*}
\end{lemma}

\begin{lemma}\label{lem_model}
$M \in \M_k$ for all $k$ with probability at least $1 - \delta/2$.
\end{lemma}

\begin{lemma}\label{lem_explore}
The number of exploration phases is bounded by $\expmax$ with probability at least $1 - \delta/2$.
\end{lemma}
The proofs of the lemmas are delayed while we apply them to prove Theorem \ref{thm_upper}.

\begin{proofof}{ of Theorem \ref{thm_upper}}
By Lemma \ref{lem_model}, $M \in M_k$ for all $k$ with probability $1 - \delta/2$. By Lemma \ref{lem_explore} we have that the number
of exploration phases is bounded by $\expmax$ with probability $1 - \delta / 2$.
Now if $t$ is not in a delaying or exploration phase and $M \in \M_k$ then by Lemma \ref{lem_optimism}, $\pi$ is nearly-optimal.
Finally note that the number of updates is bounded by $\umax$ and so the number of time-steps in delaying phases is
at most $H\umax$. Therefore UCRL is nearly-optimal for all but
$H\umax + H\expmax$
time-steps with probability $1 - \delta$.
\end{proofof}
We now turn our attention to proving Lemmas \ref{lem_optimism}, \ref{lem_model} and \ref{lem_explore}. Of these, only Lemma \ref{lem_explore}
presents a substantial challenge.

\begin{proofof}{ of Lemma \ref{lem_optimism}}
For part 1 we note that for $\iota \in \mathcal I$ the knownness of a state/action pair at level $\iota$ satisfies $\kappa \in \K$. Since
the knownness index for each $\iota$ is non-decreasing and an update only occurs when an index is increased, the total number
of updates is bounded by $\umax := \constumax$.

The proof of part 2 is closely related to the approach taken by \citet{LS08}. Recall that $\tM$ is chosen optimistically by extended value
iteration. This generates an MDP, $\tM$, such that $V^*_{\tM}(s) \geq V^*_{\tM'}(s)$ for all $\tM' \in \mathcal M_k$. Since we have assumed
$M \in \mathcal M_k$ we have that $\tV^{\pi_k}(s) \equiv V^*_{\tM}(s) \geq V_M^*(s)$. Therefore $\tV^{\pi_k}(s_t) - V^{\pi}(s_t) > \epsilon$.
Finally note that $t$ is a non-delaying time-step and so policy $\pi$ will remain stationary and equal to $\pi_k$ for at least $H$ time-steps.
Using the definition of the horizon, $H$, we have that $|V^{\pi}(s_t) - V^{\pi_k}(s_t)| < \epsilon/2$. Therefore
$\tV^{\pi_k}(s_t) - V^{\pi_k}(s_t) > \epsilon/2$ as required.
\end{proofof}

\begin{proofof}{ of Lemma \ref{lem_model}}
In the previous lemma we showed that there are at most $\umax$ updates. Therefore we only need to
check $M \in \M_k$ for each $k$ up to $\umax$. Fix an $(s, a)$ pair and apply the best of either Bernstein or Hoeffding inequalities
to show that $|\hp_{s,a}^{\sap} - p_{s,a}^{\sap}| \leq \Call{ConfidenceInterval}{\hp_{s,a}^{\sap} - p_{s,a}^{\sap}, n(s,a))}$
with probability $1 - \delta_1$.
Setting $\delta_1 := {\delta \over 2|\SA|\umax} \equiv {\delta \over 2|\SA|^2|\KI|}$ and applying the union bound completes the proof.
\end{proofof}
We are now ready to work on Lemma \ref{lem_explore}. The proof follows from two lemmas:
\begin{enumerate}
\item If $t$ is the start of an exploration phase then there exists a $(\ki)$ such that $|K_t(\ki)| > \kappa$.
\item If $|K_t(\ki)| > \kappa$ for sufficiently many $t$ then sufficient information is gained that some state/action pair
must have an increase in knownness.
\end{enumerate}
\begin{lemma}\label{lem_hard}
Let $t$ be a non-delaying time-step and assume $M \in \M_k$.
If $|K_t(\ki)| \leq \kappa$ for all $\ki \in \K$ then
$|\tV^{\pi_k}(s_t) - V^{\pi_k}(s_t)| \leq \epsilon / 2$.
\end{lemma}
The full proof is long, technical and has been relegated to Appendix \ref{app_tech}. We provide a sketch, but first we
need some useful results about MDPs and the differences in value functions.

\begin{lemma}\label{lem_transform}
Let $M$ and $\tM$ be two Markov decision processes differing only in transition probabilities and $\pi$ be a stationary policy
then
\eq{
V^\pi(s_t) - \tV^\pi(s_t) \s{=} \gamma \sum_{s} w_t(s) (p_{s,\pi} - \tp_{s,\pi}) \cdot \tV^\pi.
}
\end{lemma}
\begin{proofsketchof}{}
Drop the $\pi$ superscript and write
$V(s_t) = r(s_t) + \gamma \sum_{s_{t+1}} p_{s_t,\pi}^{s_{t+1}} V(s_{t+1})$. Then
$V(s_t) - \tV(s_t) \s{=} \gamma [p_{s_t,\pi} - \tp_{s_t,\pi}] \cdot \tV + \gamma \sum_{s_{t+1}}
p_{s_t,\pi}^{s_{t+1}} [V(s_{t+1}) - \tV(s_{t+1})]$.
The result is obtained by continuing to expand the second term of the right hand side.
\end{proofsketchof}

\begin{lemma}\label{lem_estimate_p_hp}
If $M \in \M_k$ at time-step $t$ and $\tV := \tV^{\pi_k}$ then
\eq{
|(p_{s,\pi_k} - \tp_{s,\pi_k}) \cdot \tV| \leq \sqrt{8L_1 \ts^{\pi_k}(s)^2 \over n_t(s)} + {2 \over 1 - \gamma}\left({L_1 \over n_t(s)}\right)^{3/4} + {4L_1 \over 3n_t(s)(1 - \gamma)},
}
where $\ts^{\pi_k}(s)^2 := \tp_{s,a}\cdot \tV^2 - \left[\tp_{s,a} \cdot \tV\right]^2$.
\end{lemma}
The idea is to note that $M, \tM$ are in $\M_k$ and apply the definition of the confidence intervals. The full
proof is subsumed in the proof of the more general Lemma \ref{lem_estimate_p_hp_hard} in Appendix \ref{app_hard}.
The following lemma bounds the expected total discounted local variance.
\begin{lemma}\label{lem_bound}
For any stationary $\pi$ and $\tM$,
$\sum_{s \in S} \tilde w_t(s) \tilde \sigma^\pi(s)^2 \leq {1 \over \gamma^2(1 - \gamma)^2}$.
\end{lemma}
See the paper of \citet{Sob82} for a proof.

\begin{proofsketchof}{ of Lemma \ref{lem_hard}}
For ease of notation we drop references to $\pi_k$. We approximate
$w(s) \approx \tilde w(s)$ and $|(p_{s,\pi_k} - \tp_{s,\pi_k})\cdot \tV| \lesssim \sqrt{L_1\tilde \sigma(s)^2 \over n(s)}$.
Using Lemma \ref{lem_transform}
\eqn{
\label{eq-s1} |\tV(s_t) - V(s_t)|
&\s{\equiv} \left|\gamma \sum_{s \in S} w_t(s) (p_{s,\pi_k} - \tp_{s,\pi_k}) \cdot \tV\right|
\s{\lesssim} \left| \sum_{s \in X} w_t(s) (p_{s,\pi_k} - \tp_{s,\pi_k}) \cdot \tV \right| \\
\label{eq-s2} &\s{\lesssim} \sum_{s \in X} w_t(s) \sqrt{L_1 \ts(s)^2 \over n(s)}
\s{\lesssim} \sum_{\ki \in \KI} \sum_{s \in K(\ki)} \sqrt{L_1 \tilde w_t(s) \ts(s)^2 \over \kappa m} \\
\label{eq-s3} &\s{\leq} \sum_{\ki \in \KI} \sqrt{ {L_1|K(\ki)| \over \kappa m} \sum_{s \in K(\ki)} \tilde w_t(s) \ts(s)^2}
\s{\leq} \sqrt{L_1|\KI| \over m\gamma^2(1 - \gamma)^2},
}
where in \eqr{eq-s1} we used Lemma \ref{lem_transform} and the fact that states not in $X$ are visited very infrequently. In \eqr{eq-s2}
we used the approximations for $(p - \tp) \cdot \tV$, the definition of $K(\ki)$ and the approximation $w \approx \tilde w$.
In \eqr{eq-s3} we used the Cauchy-Schwartz inequality,\footnote{$\left|\left<\mathds{1},v \right>\right| \leq \twonorm{\mathds{1}} \twonorm{v}$.}
the fact that $\kappa \geq |K(\ki)|$ and Lemma \ref{lem_bound}.
Substituting
$m \s{:=} \constm$
completes the proof. The extra terms in $m$ are needed to cover the errors in the
approximations made here.
\end{proofsketchof}
The full proof requires formalising the approximations made at the start of the sketch above. The second approximation is comparatively easy while
the showing that $w(s) \approx \tilde w(s)$ requires substantial work.

The following lemmas are used to show that $|K_t(\ki)|$ cannot be larger than $\kappa$ for too many time-steps with high probability.
Combined with Lemma \ref{lem_hard} above this will be sufficient to bound the number of exploration phases.
Let $t$ be the start of an exploration phase and define $\nu_t(s)$ to be the number of visits
to state $s$ within the next $H$ time-steps. Formally, $\nu_t(s) := \sum_{i=t}^{t+H-1} \ind{s_t = s}$.
\begin{lemma}\label{lem_visit_weights}
Let $t$ be the start of an exploration phase and $w_t(s) \geq \wmin$ then
$\E \nu_t(s) \geq w_t(s) / 2$.
\end{lemma}

\begin{proofsketchof}{}
Use the definition of the horizon to show that $w_t(s)$ is not much larger than a bounded-horizon version.
Compare $\E\nu_t(s, \pi_t(s))$ and the definition of $w_t(s)$.
\end{proofsketchof}{}
\begin{lemma}\label{lem_actual_visits}
Let $N$ be as in Appendix \ref{app_const}.
If $|K_{t_i}(\ki)| > \kappa$ for $4N$ exploration phases $t_1, t_2, \cdots, t_{4N}$ then
$\sum_{i=1}^{4N} \sum_{s \in K_{t_i}(\ki)} \nu_{t_i}(s,\pi(s)) \s{\geq} N \kappa w_{\iota}$
with probability at least $1 - \delta_1$.
\end{lemma}

\begin{proof}
As in the previous proof we drop $\pi$ superscripts and denote $K_i := K_{t_i}(\ki)$. Define
\eq{
\nu_{i} &\s{:=} \sum_{s,a \in K_i} \nu_{t_i}(s) & \E \nu_i &\s{=} \sum_{s \in K_i} \E \nu_{t_i}(s).
}
Now $|K_i| > \kappa$ and so by
Lemma \ref{lem_visit_weights} we have
$\E \nu_i \s{\geq} \kappa w_{\iota} / 2$.
We now prepare to use Bernstein's inequality.
Let $X_i = \nu_i - \E \nu_i$, $\mu := {1 \over 4N} \sum_{i=1}^{4N} \E \nu_i$ and
$\sigma^2 := {1 \over 4N} \sum_{i=1}^{4N} \Var X_i$ then
\eq{
\P{\sum_{i=1}^{4N} \nu_i \s{\leq} Nw_\iota \kappa} &\s{\leq} \P{\sum_{i=1}^{4N} \nu_i \s{\leq} \sum_{i=1}^{4N} \E \nu_i / 2} \\
&\s{=} \P{\sum_{i=1}^{4N} [\nu_i - \E \nu_i] \s{\leq} - \sum_{i=1}^{4N} \E \nu_i / 2}
\s{\leq} 2\exp\left(-{4N\mu^2 \over 8\sigma^2 + {16\mu \over 3(1 - \gamma)}}\right).
}
Setting this equal to $\delta_1$ and solving for $4N$ gives
\eq{
4N \s{\geq} {8\sigma^2 + {16\mu \over 3(1 - \gamma)} \over \mu^2} \log{2 \over \delta_1}
\s{=} \left[{8\sigma^2 \over \mu^2} + {16\over 3\mu(1 - \gamma)}\right] \log{2 \over \delta_1}.
}
Naively bounding $\sigma^2 / \mu^2 \leq 1/((1 - \gamma)\mu)$ and noting that $\mu \geq \wmin / 2$ leads to
\eq{
4N \s{\geq} {14 |\SA| \over \epsilon (1 - \gamma)^2} \log{2 \over \delta_1}.
}
Since $4N$ satisfies this, the result is complete.
\end{proof}

\begin{proofof}{ of Lemma \ref{lem_explore}}
We proceed in two stages. First we bound the total number of {\it useful} visits before $|K(\ki)| \leq \kappa$. We then show this number of
visits occurs after $\aO(m)$ exploration phases with high probability.

\subsubsect{Bounding the number of useful visits}
A visit to state/action pair $(s, a)$ in time-step $t$ is {\it $(\ki)$-useful} if $\kappa(\iota, n_t(s, a)) = \kappa$.
Fixing a $(\ki)$ we bound the number of $(\ki)$-useful visits to state/action pair $(s, a)$.
Suppose $t_1 < t_2$ and $\kappa(\iota,n_{t_1}(s,a)) = \kappa$ and $n_{t_2}(s, a) - n_{t_1}(s,a) \geq m w_\iota (2 \kappa + 2)$ then
$\kappa(\iota,n_{t_3}(s, a)) > \kappa$ for all $t_3 \geq t_2$.
Therefore for each $(\ki)$ pair there at most $6|\SA|mw_\iota \kappa$ visits that are $(\ki)$-useful.

\subsubsect{Bounding the number of exploration phases}
Let $N := \constN$ and
$t$ be the start of an exploration phase. Therefore $\tV^{\pi_k}(s_t) - V^{\pi_k}(s_t) > \epsilon / 2$
and so by Lemma \ref{lem_hard} there exists a $(\ki) \in \K$ such that
$|S| \geq |K(\ki)| > \kappa$.
If $|K_{t_i}(\kappa, \iota)| > \kappa$ at
the start of $4N$ exploration phases, $t_1, t_2, \cdots, t_{4N}$ then by Lemma \ref{lem_actual_visits}
\eq{
\P{\sum_{i=1}^{4N} \sum_{s,a \in K_{t_i}(\kappa,\iota)} v_{t_i}(s, a) \leq N w_\iota \kappa} \s{\leq} \delta_1.
}
Therefore by the union bound there are at most $\expmax := \constexpmax$ exploration phases with probability
$1 - \delta_1 |\KI| \equiv 1 - |\KI|\constdeltaone > 1 - \delta/2$.
\end{proofof}

\vspace{-2mm}
\section{Eliminating the Assumption}\label{sec_extension}

The upper bound in the previous section could only be proven using Assumption \ref{ass}. In this section we describe a possible approach to
generalising the proof and why this may be non-trivial. In the work above we used the assumption to bound
$(p_{s,\pi} - \tp_{s,\pi}) \cdot \tV^* \lesssim \sqrt{L_1 \ts^\pi(s)^2 / n}$. A natural approach to generalising this comes from
Bernstein's inequality (Theorem \ref{thm_bernstein}). If $V^\pi \in \R^{|S|}$ is a value function independent of $\hp$ then Bernstein's inequality
can be used to show that $(p_{s,\pi} - \hp_{s,\pi}) \cdot V^\pi \lesssim \sqrt{L_1 \sigma^\pi(s)^2 / n}$.
This suggests we adjust our model class by letting $\pi := \tilde \pi^*$ and changing the condition to $(\tp_{s,\pi} - \hp_{s,\pi}) \cdot \tV^* \lesssim \sqrt{L_1 \ts^\pi(s)^2 / n}$.
We might then bound
$(p_{s,\pi} - \tp_{s,\pi}) \cdot \tV^* \equiv (p_{s,\pi} - \hp_{s,\pi}) \cdot \tV^* + (\hp_{s,\pi} - \tp_{s,\pi}) \cdot \tV^*$. The right term
is then bounded by the conditions on the model class and the left term can perhaps be bounded by noting that $p_{s,\pi}$ is the true probability
distribution. Unfortunately, there are a few problems with this approach:
\begin{enumerate}
\item Bounding $(p_{s,\pi} - \hp_{s,\pi}) \cdot \tV^*$ does not result in a bound in terms of $\ts^\pi(s)^2$.
This issue can be solved by again applying Bernstein's inequality to bound $(p_{s,a} - \hp_{s,a})\cdot {\tV^{*^2}}$.
\item The value $\tV^*$ is {\it not} in general independent of $\hp$. This is because $\tM$ must be chosen to satisfy the conditions on
$(\hp_{s,\pi} - \tp_{s,\pi})\cdot \tV^*$, which depends on $\hp$. This dependence violates the conditions of Bernstein's inequality when
trying to bound $(p_{s,\pi} - \hp_{s,\pi}) \cdot \tV^*$.
The dependence is intuitively quite weak, but nevertheless presents problems
for rigorous proof.
\item The last problem is that extended value iteration is no longer a trivial operation (even granting infinite computation). The problem is that
the condition $(p_{s,a} - \hp_{s,a}) \cdot V^*$ is not local to $(s, a)$, it also depends on the choice of $p_{s',a'}$ for $(s',a') \in \SA$. This
complication is probably resolvable, but the formal demonstration of extended value iteration is no longer so easy.
\end{enumerate}

\subsubsect{Progress}
The first issue above can be solved, as remarked, by bounding $(p_{s,a} - \hp_{s,a}) \cdot {\tV^{*^2}}$ using another Bernstein inequality. The problem
here is that this condition must now be added to the definition of the model class. The second issue is non-trivial and we cannot claim to have made
progress there.
We did manage to show that extended value iteration can be extended to the case where the only constraints take the form
$(\tp_{s,\pi} - \hp_{s,\pi}) \cdot \tV^* \lesssim \sqrt{L_1 \ts^\pi(s)^2 / n}$.
In this case it can be shown the existence of a globally
optimistic MDP. Unfortunately if you add constraints on higher moments, $(\tp_{s,\pi} - \hp_{s,\pi}) \cdot \tV^2$ then results become substantially
more complex. Note that in the complete proof of Lemma \ref{lem_hard} we used higher moments still, but this is not required. Lemma \ref{lem_hard}
can be proven using only bounds on $(\tp - \hp) \cdot \tV^*$ and $(\tp - \hp)\cdot {\tV^{*^2}}$.

\section{Lower PAC Bound}\label{sec_lower}

We now turn our attention to proving a matching lower bound. The approach is
similar to that of \citet{Str09}, but we make two refinements
to improve the bound to depend on $1/(1 - \gamma)^3$ and remove the policy restrictions. The first is to add a delaying state where no information
can be gained, but where an algorithm may still fail to be PAC. The second is more subtle and will be described in the proof.

\begin{definition}
A non-stationary policy is a function $\pi:S^* \to A$.
\end{definition}

\begin{theorem}\label{thm_lower}
Let $\pi$ be a (possibly non-stationary) policy depending on $S, A, r, \gamma, \epsilon$ and $\delta$,
then there exists a Markov decision process $\Mh$ such that
$V^*(s_t) - V^\pi(s_t) > \epsilon$ for at least $N$ time-steps with probability at least $\delta$ where
\eq{
N \s{:=} {c_1 |\SA| \over \epsilon^2(1-\gamma)^3} \log {c_2 \over\delta}
}
and $c_1,c_2 > 0$ are independent of the policy $\pi$ as well as all inputs $S, A, \epsilon, \delta, \gamma$.
\end{theorem}
The proof can found in Appendix \ref{app_lower}, but we give the counter-example MDP and intuition.

\setlength{\intextsep}{0pt}
\begin{wrapfigure}[14]{r}{5.3cm}
\topsep=0.0cm
\vspace{0.3cm}
\hspace{-0.3cm}
\small
\begin{tikzpicture}[->,>=stealth',shorten >=1pt,auto,node distance=2.7cm, semithick]
\tikzstyle{state} = [circle,draw,minimum width=1.0cm, minimum height=1.0cm]
\node[state] (error) {\small $\stackrel{1} {r = 0}$};
\node[state] (minus) [right of=error] {\small $\stackrel{\ominus} {r = 0}$};
\node[state] (plus) [below of=error] {\small $\stackrel{\oplus} {r = 1}$};
\node[state] (hard) [right of=plus] {\small $\stackrel{0} {r = 0}$};
\path (hard) edge node[swap] {$1 - p$} (error)
      (hard) edge[loop below] node {$p := 1/(2 - \gamma)$} (hard)
      (error) edge node {${1 \over 2} - \epsilon(a)$} (minus)
      (error) edge node[swap] {${1 \over 2} + \epsilon(a)$} (plus)
      (minus) edge[loop above] node[swap] {$q := 2 - 1/\gamma$} (minus)
      (minus) edge node {$1-q$} (hard)
      (plus) edge[loop below] node[swap] {$q$} (plus)
      (plus) edge node[swap] {$1 - q$} (hard)
;
\end{tikzpicture}
\begin{center}
Figure 1: Hard MDP
\end{center}
\end{wrapfigure}
\subsubsect{Counter Example}
We prove Theorem \ref{thm_lower} for a class of MDPs
where $S = \left\{0,1,\oplus,\ominus\right\}$ and $A = \left\{1,2,\cdots, |A|\right\}$. The rewards and transitions for a single
action are depicted in the diagram on the right where $\epsilon(a^*) = 16\epsilon(1 - \gamma)$ for some $a^* \in A$ and $\epsilon(a) = 0$
for all other actions.
Some remarks:
\begin{enumerate*}
\item States $\oplus$ and $\ominus$ are almost completely absorbing and confer maximum/minimum rewards respectively.
\item The transitions are independent of actions for all states except state $1$. From this state, actions lead
uniformly to $\oplus$/$\ominus$ except for one action, $a^*$, which has a slightly higher probability of transitioning
to state $\oplus$. Thus $a^*$ is the optimal action in state $1$.
\item State $0$ has an absorption rate such that, on average, a policy will stay there for $1/(1 - \gamma)$ time-steps.
\end{enumerate*}

\subsubsect{Intuition}
The MDP above is very bandit-like in the sense that once a policy reaches state $1$ it should choose the action most likely to lead to
state $\oplus$ whereupon it will either be rewarded or punished (visit state $\oplus$ or $\ominus$). Eventually it will return to state $1$
when
the whole process repeats. This suggests a PAC-MDP algorithm can be used to learn the bandit with $p(a) := p_{1,a}^{\oplus}$.
We can then make use of a theorem of \citet{MT04} on bandit sample-complexity to show
that the number of times $a^*$ is not selected is at least
\eqn{\label{eq-lower}
\aO\left({1 \over \epsilon^2(1 - \gamma)^2} \log{1 \over \delta}\right).
}
Improving the bound to depend on $1/(1 - \gamma)^3$ is intuitively easy, but technically somewhat annoying. The idea is to
consider the value differences in state $0$ as well as state $1$. State $0$ has the following properties:
\begin{enumerate*}
\item The absorption rate is sufficiently large that any policy remains in state $0$ for around $1/(1 - \gamma)$ time-steps.
\item The absorption rate is sufficiently small that the difference in values due to bad actions planned in state $1$ still matter while in state $0$.
\end{enumerate*}
While in state $0$ an agent cannot make
an error in the sense that $V^*(0) - Q^*(0, a) = 0$ for all $a$. But we are measuring $V^*(0) - V^\pi(0)$ and so an agent can be penalised if
its policy upon reaching state $1$ is to make an error.
Suppose the
agent is in state $0$ at some time-step before moving to state $1$ and making a mistake. On average it will stay in state $0$ for roughly $1/(1-\gamma)$
time-steps during which time it will plan a mistake upon reaching state $1$. Thus the bound in \eqr{eq-lower} can be multiplied
by $1/(1 - \gamma)$.
The proof is harder because an agent need not plan to make a mistake in all future time-steps when reaching state $1$ before eventually doing so
in one time-step.
Note that \citet{Str09} proved their theorem for a specific class of policies while Theorem \ref{thm_lower} holds for all policies.

\section{Conclusion}

\subsubsect{Summary}
We presented matching upper and lower bounds on the number of time-steps when a reinforcement learning algorithm can be nearly-optimal with high
probability. While the lower bound is completely general, the upper bound depends on the assumption that there are at most two next-states for each
state/action pair. This assumption aside, the new upper bound improves on the previously best known bound of \citet{Aue11}. If the assumption is
dropped then the new proof can be used to construct an algorithm that is better than the bound of \citet{Aue11} in terms of $1/(1 - \gamma)$, but
worse in $|S|$. The lower bound, which comes without assumptions, improves on the work of \citet{Str09} by being both larger and more general. The
class of MDPs used for the counter-example do satisfy Assumption \ref{ass} and so the upper and lower bounds now match in this restricted case.

\subsubsect{Running Time}
We did not analyze the running time of our version of UCRL, but expect analysis similar to that of \citet{LS08} can be used to show that
UCRL can be approximated to run in polynomial time with no cost to sample-complexity.

\acks{Thanks to Peter Sunehag for his careful reading and useful suggestions.}


\addcontentsline{toc}{section}{\refname}
\begin{small}

\end{small}

\appendix
\section{Proof of Lower PAC Bound}\label{app_lower}

The proof makes use of a simple form of bandit and Theorem \ref{thm_bandit}, which lower
bounds the sample-complexity of bandit algorithms. We need some new notation required for non-stationary policies
and bandits.

\subsubsect{History Sequences}
We write $s_{1:t} = s_1, s_2, \cdots, s_t$ for the history sequence of length $t$. Histories can be concatenated, so
$s_{1:t}\oplus = s_1,s_2,\cdots, s_t,\oplus$ where $\oplus \in S$.

\subsubsect{Bandits}
An {\it $A$-armed bandit} is a vector $p : A \to [0,1]$. A policy interacts with a bandit sequentially. In time-step $t$ some arm $a_t$ is
played whereupon the policy
receives reward $1$ with probability $p(a)$ and reward $0$ otherwise. This is repeated over all time-steps. More formally, a bandit policy
is a function $\pi : \left\{0,1\right\}^* \to A$.  The optimal arm is defined $a^* := \argmax_a p(a)$.
A policy dependent on $\epsilon, \delta$ and $A$ has sample-complexity $T:=T(A, \epsilon, \delta)$ if for all bandits the arm chosen on
time-step $T$ satisfies $p(a^*) - p(a_T) \leq \epsilon$ with probability at least $1 - \delta$.

\begin{theorem}[Mannor and Tsitsiklis, 2004]\label{thm_bandit}
There exist positive constants $c_1$, $c_2$, $\epsilon_0$, and $\delta_0$, such that for every $A \geq 2$, $\epsilon \in (0, \epsilon_0)$ and
$\delta \in (0, \delta_0)$ there exists a bandit $p \in [0, 1]^A$ such that
\eq{
T(A, \epsilon, \delta) \s{\geq} c_1 {|A| \over \epsilon^2} \log {c_2 \over \delta}
}
with probability at least $\delta$.
\end{theorem}

\begin{remark}
The bandit used in the proof of Theorem \ref{thm_bandit} satisfies $p(a) = {1 \over 2}$ for all $a$ except $a^*$ which has
$p(a^*) := {1 \over 2} + \epsilon$.
\end{remark}
We now prepare to prove Theorem \ref{thm_lower}. For the remainder of this section let $\pi$ be an arbitrary policy
and $\Mh$ be the MDP of Figure 2. As in previous work we write $V^\pi := V^\pi_{\Mh}$.
The idea of the proof will be to use Theorem \ref{thm_bandit} to  show that $\pi$ cannot be approximately correct in state $1$ too often. Then
use this to show that while in state $0$ before-hand it is also not approximately correct.

\begin{definition}
Let $s_{1:\infty} \in S^\infty$ be the sequence of states seen by policy $\pi$ and for arbitrary history $s_{1:t}$ let
\eq{
\Delta(s_{1:t}) := V^*(s_{1:t}) - V^\pi(s_{1:t}).
}
\end{definition}

\begin{lemma}\label{lem_tech}
If $\gamma \in (0, 1)$, $p := 1/(2 - \gamma)$ and $q := 2 - 1/\gamma$ then
\eq{
p^{1 \over 4(1 - \gamma)}  \s{>} 3/4 \quad\text{ and } \quad
\sum_{t=0}^\infty p^t (1 - p) \gamma^t \s{=} {1 \over 2}.
}
\end{lemma}

\begin{proofsketchof}{}
Both results follow from the geometric series and easy calculus.
\end{proofsketchof}
The following lemma lower-bounds $\Delta(s_{1:t})$ if sub-optimal action $a \neq a^*$ is taken in state $1$.
\begin{lemma}\label{lem_tech2}
Let $s_{1:t}$ be a history such that $s_t = 1$ and $a := \pi(s_{1:t}) \neq a^*$ then
\eq{
\Delta(s_{1:t}) \geq 8\epsilon.
}
\end{lemma}

\begin{proof}
The result essentially follows from the definition of the value function.
\eq{
\Delta(s_{1:t}) &\equiv V^*(s_{1:t}) - V^\pi(s_{1:t}) \\
&= \gamma \left[p_{1,a^*}^{\oplus} V^*(s_{1:t}\oplus) + p_{1,a^*}^{\ominus} V^*(s_{1:t}\ominus)\right]
- \gamma \left[p_{1,a}^{\oplus} V^\pi(s_{1:t}\oplus) + p_{1,a}^{\ominus}V^\pi(s_{1:t}\ominus)\right] \\
&= {\gamma \over 2} \left[V^*(s_{1:t}\oplus) - V^\pi(s_{1:t}\oplus) + V^*(s_{1:t}\ominus) - V^\pi(s_{1:t}\ominus)\right] + \gamma\epsilon(a^*) V^*(s_{1:t}\oplus) \\
&\geq 8\epsilon,
}
where we used the definition of the value function and MDP, $\Mh$.
\end{proof}
We now define time-intervals where the policy is in state $0$. Recall we chose the absorption in this state such that the expected
number of time-steps a policy remains there is approximately $1/(1 - \gamma)$. We define the intervals starting
when a policy arrives in state $0$ and ending when it leaves to state $1$.
\begin{definition}
Define $t_1^0 := 1$ and
\eq{
t_i^0 &:= \min \left\{t : t > t_{i-1} \wedge s_t = 0 \wedge s_{t-1} \neq 0 \right\} &
t_i^1 &:= \min\left\{t - 1 : s_t = 1 \wedge t > t_i^0\right\}.
}
Define the intervals $I_i := [t_i^0, t_i^1] \subseteq \N$. We call interval $I_i$ the $i$th {\it phase}.
\end{definition}
Note the following facts:
\begin{enumerate*}
\item Since all transition probabilities are non-zero, $t^0_i$ and $t^1_i$ exist for all $i \in \N$ with probability $1$.
\item $|I_i|$ is the number of time-steps spent in state $0$ before moving to state $1$.
\item The values $|I_i|$ are independent of $\pi$ and each other.
\end{enumerate*}
\begin{definition}
Suppose $t \in \N$ and $s_t = 0$ and define the {\it weight} of action $a$, $w_t(a)$ by
\eq{
w_t(a) := \sum_{k=0}^\infty p^k (1 - p) \gamma^k \ind{\pi(s_{1:t}0^k1) = a}.
}
\end{definition}

\begin{lemma}\label{lem_weights2} $\sum_{a \in A} w_t(a) = {1 \over 2}$ for all $t$ where $s_t = 0$.
\end{lemma}

\begin{proof}
We use Lemma \ref{lem_tech}.
\eq{
\sum_{a \in A} w_i(a) &\equiv \sum_{a \in A} \sum_{k=0}^\infty p^k(1 - p) \gamma^k \ind{\pi(s_{1:t}0^k1) = a} \\
&= \sum_{k=0}^\infty p^k(1 - p)\gamma^k = {1 \over 2}
}
as required.
\end{proof}

\begin{definition}
Define random variables $A_i$ and $X_i$ by
\eq{
A_i &\s{:=} \ind{|I_i| \geq 1/[16(1 - \gamma)] \wedge \sum_{a \neq a^*} w_{t^0_i}(a) \geq 1/4} &
X_i &\s{:=} \ind{|I_i| \geq 1/[4(1 - \gamma)]}
}
\end{definition}
Intuitively, $X_i$ is the event that the $i$th phase lasts at least $1/[4(1 - \gamma)]$ time-steps. $A_i$
is the event that the $i$th phase lasts at least $1/[16(1 - \gamma)]$ time-steps and the combined weight
of sub-optimal actions at the start of a phase is at least $1/4$.
The following lemma shows that at least two thirds of all phases have $X_i = 1$ with high probability.
\begin{lemma}\label{lem_time}
For all $n \in \N$, $\P{\sum_{i=1}^n X_i \leq {2 \over 3} n} \leq 2e^{-n/72}$.
\end{lemma}

\begin{proof}
Preparing to use Hoeffding's bound,
\eq{
\P{X_i = 1} := \P{|I_i| \geq 1/[4(1 - \gamma)]} = p^{1/[4(1 - \gamma)]} > 3/4,
}
where we used the definitions of $X_i$, $I_i$ and Lemma \ref{lem_tech}. Therefore $\E X_i > 3/4$.
\eq{
\P{\sum_{i=1}^n X_i \leq {2 \over 3}n} &\leq \P{\sum_{i=1}^n X_i \leq {1 \over 12}n + n \E X_i}
= \P{\sum_{i=1}^n X_i - \E X_i \leq {1 \over 12} n} \leq 2 e^{-n / 72}
}
where we applied basic inequalities followed by Hoeffding's bound.
\end{proof}

\begin{lemma}\label{lem_weights}
If $\gamma > {3 \over 4}$ and $\sum_{a \neq a^*} w_t(a) \geq {1 \over 4}$ then
$\sum_{a\neq a^*} w_{t+k}(a) \geq {1 \over 8}$ for all $t \in \N$ and $k$ satisfying $0 \leq k \leq 1/[16(1 - \gamma)]$.
\end{lemma}

\begin{proof}
Working from the definitions.
\eq{
{1 \over 4} \leq \sum_{a \neq a^*} w_{t^0_i}(a) &\equiv \sum_{j=0}^\infty p^j(1-p) \gamma^j \ind{\pi(s_{1:t^0_i}0^j) \neq a^*} \\
&= \sum_{j=0}^{k-1} p^j(1-p) \gamma^j \ind{\pi(s_{1:t^0_i}0^j) \neq a^*} + p^k \gamma^k \sum_{a \neq a^*} w_a(s_{1:t^0_i} 0^k) \\
&\leq (1 - p) \sum_{j=0}^{k-1} p^j \gamma^j + p^k \gamma^k \sum_{a \neq a^*} w_a(s_{1:t^0_i} 0^k)
}
Rearranging,
setting $0 \leq k \leq 1/[16(1-\gamma)]$ and using the geometric series completes the proof.
\end{proof}
So far, none of our results have been especially surprising. Lemma \ref{lem_time} shows that at least two thirds of all phases have length exceeding $1/[4(1 - \gamma)]$
with high probability. Lemma \ref{lem_weights} shows that if at the start of a phase $\pi$ assigns a high weight to the sub-optimal actions, then it
does so throughout the entire phase. The following lemma is more fundamental. It shows that the number of phases where $\pi$ assigns a high
weight to the sub-optimal actions is of order ${1 \over \epsilon^2(1 - \gamma)^2} \log{1 \over \delta}$ with high probability.

\begin{lemma}\label{lem_bandit}
Let $N := {c_1 A \over \epsilon^2 (1 - \gamma)^2} \log {c_2 \over \delta}$ with constants as in Theorem \ref{thm_bandit} then
\eq{
\left|\left\{i : \sum_{a \neq a^*} w_{t^0_i}(a) > {1 \over 4} \wedge i < 2N + 1\right\} \right| > N
}
with probability at least $\delta$.
\end{lemma}
\newcommand{\abest}{a_{\operatorname{best}}}
The idea is similar to that in \citep{Str09}. Assume a policy exists that doesn't satisfy the condition above and then
use it to learn the bandit defined by $p(a) := p_{1,a}^{\oplus}$.
\begin{proof}
Let $p(a) := p_{1,a}^{\oplus}$ be a bandit and use $\pi$ to learn bandit $p$ using Algorithm 2 below,
which returns an action $\abest$ defined
as
\eq{
\abest := \argmax_{a} \sum_{i=1}^{2N} \bar a_i, \quad \bar a_i := \argmax_{a'} w_{t_i^0}(a')
}
By Theorem \ref{thm_bandit}, the strategy in Algorithm 2 must fail with probability at least $\delta$. Therefore with
probability at least $\delta$, $\abest \neq a^*$. However $\abest$ is defined as the majority action of all the $\bar a_i$
and so for at least $N$ time-steps $\bar a_i \neq a^*$.
Suppose $w_{t_i^0}(a) > {1 \over 4}$, then by Lemma \ref{lem_weights2}, $\sum_{a \neq a^*} w_{t_i^0}(a) < {1 \over 4}$ and
$\bar a_i \equiv \argmax_{a} w_{t_i^0}(a) = a^*$.
This implies that with probability $\delta$, for at least $N$ time-steps $\sum_{a \neq a^*} w_{t^0_i}(a) > {1 \over 4}$
as required.
\end{proof}

\begin{algorithm}[H]
\caption{Learn Bandit}
\begin{algorithmic}
\State $t = 1$, $s_t = 0$, $k = 0$
\Loop
\State $a_t = \pi(s_{1:t})$
\If{$s_t=1$}
\State $r \sim p(a_t)$ \Comment{sample from bandit}
\If{$r = 1$}
\State $s_{t+1} = \oplus$
\Else
\State $s_{t+1} = \ominus$
\EndIf
\State $k = k + 1$
\If{$k = 2N$}
\State $a_{\operatorname{best}} = \argmax_a \sum_{i=1}^{2N} \ind{a = \argmax_{a'} w_{t^0_i}(a')}$
\State {\bf exit}
\EndIf
\Else
\State $s_{t+1} \sim p_{s_t,a_t}$ \Comment{sample from MDP}
\EndIf
\EndLoop
\end{algorithmic}
\end{algorithm}
\vspace{1mm}

\begin{proofof}{ of Theorem \ref{thm_lower}}
Suppose $A_i = 1$ and $0 \leq k \leq 1/[16(1 - \gamma)]$ then $s_{1:t_i^0+k} = s_{1:t_i^0}0^k$ and
\eqn{
\label{eq-lower2} \Delta(s_{1:t_i+k}) &= \sum_{t=0}^\infty {p^t(1 - p)\gamma^t} \Delta(s_{1:t_i+k}0^t1) \\
\label{eq-lower3} &\geq \sum_{t=0}^\infty {p^t(1 - p)\gamma^t} \sum_{a \neq a^*} \ind{\pi(s_{1:t_i^0+k}0^t1) = a} 8\epsilon \\
\label{eq-lower4} &\geq \sum_{a \neq a^*} w_{t_i^0 +k}(a) 8\epsilon \\
\label{eq-lower5} &\geq \epsilon,
}
where \eqr{eq-lower2} follows from the definition of $\Mh$ and the value function. \eqr{eq-lower3} by Lemma \ref{lem_tech2}. \eqr{eq-lower4} by
the definition
of $w_{t_i+k}(a)$ and \eqr{eq-lower4} by Lemma \ref{lem_weights}.
Thus for each $i$ where $A_i = 1$, policy $\pi$ makes
at least $1/[16(1 - \gamma)]$ $\epsilon$-errors. The proof is completed by showing that $A_i = 1$ for at least $N/6$ time-steps
with probability at least $\delta$,
which follows easily from Lemma \ref{lem_bandit} and Lemma \ref{lem_time}.

Dependence on $S$ is added trivially by chaining arbitrarily many such Markov decision processes together.
\end{proofof}

\begin{remark}
Dependence on $S \log S$ can possibly be added by a similar technique used by \cite{Str09}, but details could be messy.
\end{remark}

\section{Technical Results}\label{app_tech}

\begin{theorem}[Hoeffding Inequality]
Let $X_1,\cdots,X_n$ be independent $[0,1]$-valued random variables
with probability $1$. Then
\eq{
\P{\left|{1 \over n} \sum_{i=1}^n X_i-\E X_i \right| \geq \epsilon} \s{\leq} 2e^{-2\epsilon^2n}.
}
\end{theorem}

\begin{theorem}[Bernstein's Inequality \citep{Ber24}]\label{thm_bernstein}
Let $X_1,\cdots,X_n$ be independent real-valued random variables with zero mean and variance $\Var X_i = \sigma^2_i$.
If $|X_k| < c$ with probability one then
\eq{
\P{\left|{1 \over n} \sum_{i=1}^n X_i\right| \geq \epsilon} \s{\leq} 2e^{-{\epsilon^2 n \over 2 \sigma^2 + 2c \epsilon/3}},
}
where $\sigma^2 := {1 \over n} \sum_{i=1}^n \sigma^2_i$.
\end{theorem}
We can use Hoeffding and Bernstein to bound the gaps $|p - \hp|$ and $|\hp - \tp|$ we now want to combine these together in a nice way
to bound $|p - \tp|$.
\begin{lemma}\label{lem_estimate}\label{lem_estimate_bootstrap}
Let $p, \hp, \tp \in [0, 1]$ satisfy
\eq{
|p - \hp| \leq \min\left\{ CI_1, CI_2 \right\},
}
where
\eq{
CI_1 &:= \sqrt{{2p(1 - p) \over n} \log{2 \over \delta}} + {2 \over 3n} \log{2 \over \delta}
& CI_2 &:=\sqrt{{1 \over 2n} \log{2 \over \delta}}.
}
Then
\eq{
|p - \tp| &\s{\leq}
 \sqrt{{8\tp(1 - \tp) \over n} \log{2 \over \delta}} + 2\left({1 \over n} \log{2 \over \delta}\right)^{{3 \over 4}} +
{4 \over 3n} \log{2 \over \delta}
}
\end{lemma}

\begin{proof}
Using the first confidence interval
\eq{
|p - \hp| \leq \sqrt{{2p(1 - p) \over n} \log{2 \over \delta}} + {2 \over 3n} \log{2 \over \delta}
}
Assume without loss of generality that $1 - p \geq 1 - \tp$ (the case where $p \geq \tp$ is identical.
Therefore
\eq{
|p - \hp|
&\leq \sqrt{{2\tp(1 - \tp) \over n} \log{2 \over \delta}} + \sqrt{{2(p-\tp)(1 - \tp) \over n} \log{2 \over \delta}} +
{2 \over 3n} \log{2 \over \delta} \\
&\leq \sqrt{{2\tp(1 - \tp) \over n} \log{2 \over \delta}} + \sqrt{{4\sqrt{{1 \over 2n} \log{2 \over \delta}} \over n}\log{2 \over \delta}} +
{2 \over 3n} \log{2 \over \delta} \\
&= \sqrt{{2\tp(1 - \tp) \over n} \log{2 \over \delta}} + 8^{1\over 4}\left({1 \over n} \log{2 \over \delta}\right)^{{3 \over 4}} +
{2 \over 3n} \log{2 \over \delta},
}
where we used the second confidence interval and algebra. Bounding $|\hp - \tp|$ by the first confidence interval leads to
\eq{
|p - \tp| &\s{\leq}
 \sqrt{{8\tp(1 - \tp) \over n} \log{2 \over \delta}} + 2\left({1 \over n} \log{2 \over \delta}\right)^{{3 \over 4}} +
{4 \over 3n} \log{2 \over \delta}
}
as required.
\end{proof}

\section{Proof of Lemma \ref{lem_hard}}\label{app_hard}
We need to define some higher ``moments'' of the value function. This is somewhat unfortunate as it complicates
the proof, but may be unavoidable.
\begin{definition}\label{def_recurrence}
We define the space of bounded value/reward functions $\mathcal R$ by
\eq{
\mathcal R(i) \s{:=} \left\{v \in \left[0, \left({1 \over 1 - \gamma}\right)^i\right]^{|S|} \right\} \subset \R^{|S|}.
}
Let $\pi$ be some stationary policy. For $r_d \in \mathcal R(d)$ define values $V^\pi_d$ by the
Bellman equations
\eq{
V^\pi_d(s) &\s{=} r_d(s) + \gamma \sum_{s'} p_{s,\pi}^{s'} V^\pi_d(s').
}
Additionally,
\eq{
\sigma^{\pi}_d(s)^2 \s{:=} p_{s,\pi} \cdot {V^\pi_d}^2 - \left[p_{s,\pi} \cdot V^\pi_d\right]^2.
}
Note that $V_d \in \mathcal R(d+1)$ and $\sigma^2_d \in \mathcal R(2d + 2)$.
Let $r_0 \in \mathcal R(0)$ be the true reward function $r_0(s) := r(s)$ and define a {\it recurrence} by $r_{2d+2}(s) := \sigma^\pi_{d}(s)^2$.
We define $\tilde r_d$, $\hat r_d$, $\tV^\pi_d$, $\hV^\pi_d$ and $\ts_d^\pi$, $\hat \sigma_d^\pi$ similarly but where all parameters have hat/tilde.
\end{definition}
The following lemma generalises Lemma \ref{lem_estimate_p_hp}.

\begin{lemma}\label{lem_estimate_p_hp_hard}
Let $M \in \M_k$ at time-step $t$ then
\eq{
|(p_{s,\pi} - \tp_{s,\pi}) \cdot \tV^\pi_d| \leq \sqrt{8L_1 \ts^\pi_d(s)^2 \over n_t(s)} +
2\left({L_1 \over n_t(s)} \right)^{3 \over 4} {1 \over (1 - \gamma)^{d+1}} +
{4L_1 \over 3n_t(s)(1 - \gamma)^{d+1} }
}
\end{lemma}

\begin{proof}
Drop references to $\pi$ and
let $p := p_{s,\pi}^{\sap}$, $\tp := \tp_{s,\pi}^{\sap}$ and $n := n_t(s)$. Since $M, \tM \in \M_k$ then apply Lemma \ref{lem_estimate_bootstrap} to obtain
\eq{
|p - \tp| \leq \sqrt{{8L_1\tp(1 - \tp) \over n} } + 2\left({L_1 \over n} \right)^{{3 \over 4}} +
{4L_1 \over 3n}
}
Assume without loss of generality that $\tV_d(\sap) \geq \tV_d(\sam)$.
Therefore we have
\eqn{
\nonumber|(p_{s,\pi} - \tp_{s,\pi}) \cdot \tV_d| &\leq \sqrt{8L_1\tp(1 - \tp) \over n}\left(\tV_d(\sap) - \tV_d(\sam)\right) +
2\left({L_1 \over n} \right)^{3 \over 4}{1 \over (1 - \gamma)^{d+1}}  \\
\label{eq-lem-estimate} &\qquad + \;\;{4L_1 \over 3n(1 - \gamma)^{d+1}},
}
where we used Assumption \ref{ass} and the fact that $V_d \in \mathcal R_{d+1}$.
\eq{
\tp(1 - \tp)\left(\tV_d(\sap) - \tV_d(\sam)\right)^2
&= \tp(1 - \tp)\left(\tV_d(\sap)^2 + \tV_d(\sam)^2 - 2 \tV_d(\sap)\tV_d(\sam)\right) \\
&= \tp\tV_d(\sap)^2 + (1 - \tp)\tV_d(\sam)^2 - \left(\tp\tV_d(\sap) + (1 - \tp)\tV_d(\sam)\right)^2  \\
&= \ts_d(s)^2.
}
Substituting into \eqr{eq-lem-estimate} completes the proof.
\end{proof}

\begin{proofof}{ of Lemma \ref{lem_hard}}
For ease of notation we drop $\pi$ and $t$ super/subscripts.
Let
\eq{
\Delta_d := \left|\sum_{s \in S} [w(s) - \tw(s)]r_d(s)\right| \equiv |\tV_d(s_t) - V_d(s_t)|.
}
Using Lemma \ref{lem_transform}
\eq{
\Delta_d &\s{=} \gamma \left|\sum_{s\in S} w(s) (p_s - \tp_s) \cdot \tV_d\right| \\
&\s{\leq} {\epsilon \over 4(1 - \gamma)^d} + \left|\sum_{s \in X} w(s) (p - \tp) \cdot \tV_d\right| \\
&\s{\leq} {\epsilon \over 4(1 - \gamma)^d} + A_d + B_d + C_d,
}
where
\eq{
A_d &:= \sum_{s \in X} w(s) \sqrt{8L_1\ts^2_d \over n(s)} &
B_d &:= \sum_{s \in X} w(s) {4L_1 \over 3n(s)(1 - \gamma)^{d+1}} &
C_d &:= \sum_{s \in X} w(s) 2{\left({L_1 \over n(s)}\right)}^{3/4}.
}
The expressions $B_d$ and $C_d$ are substantially easier to bound than $A_d$. First we give a naive bound on $A_d$, which we use later.
\eqn{
\label{eqh-1} A_d &\s{\leq} \sum_{s \in X} \sqrt{8w(s)\ts^2_d(s)L_1\over n(s)}
\s{\equiv} \sum_{\ki \in \KI} \sum_{s \in K(\ki)} \sqrt{8w(s)\ts^2_d(s)L_1 \over n(s)} \\
\label{eqh-2}&\s{\leq} \sum_{\ki \in \KI} \sqrt{ {8L_1 |K(\ki)| \over m \kappa} \sum_{s \in K(\ki)} w(s)\ts^2_d(s)}
\s{\leq} \sum_{\ki \in \KI} \sqrt{ {8L_1 \over m} \sum_{s \in K(\ki)} w(s)\ts^2_d(s)} \\
\label{eqh-3}&\s{\leq} \sqrt{{8|\KI|L_1 \over m} \sum_{\ki \in \K} \sum_{s \in K(\ki)} w(s)\ts^2_d(s)}
\s{\leq} \sqrt{{8|\KI|L_1 \over m} \sum_{s \in X} w(s)\ts^2_d(s)} \\
\label{eqh-4}&\s{\leq} \sqrt{{8|\KI|L_1 \over m (1 - \gamma)^{2d+3}}},
}
where in \eqr{eqh-1} we used the definitions of $A_d$ and $\K$. In \eqr{eqh-2} we applied Cauchy-Schwartz and the assumption that $|K(\kappa)| \leq \kappa$.
In \eqr{eqh-3} we used Cauchy-Schwartz again and the definition of $\K$. Finally we apply the trivial bound of $\sum w(s) \ts^2_d(s) \leq 1/(1 - \gamma)^{2d+3}$.
Unfortunately this bound is not sufficient for our needs. The solution is approximate $w(s)$ by $\tilde w(s)$ and use Lemma \ref{lem_bound} to
improve the last step above.
\eqn{
\label{eqh-5} A_d &\s{\leq} \sqrt{{8|\KI|L_1 \over m} \sum_{s \in S} w(s)\ts^2_d(s)} \\
\label{eqh-6} &\s{\equiv} \sqrt{{8|\KI|L_1 \over m} \sum_{s \in S} \tilde w(s)\ts^2_d(s) +
{8|\KI|L_1 \over m} \sum_{s \in S} (w(s) - \tilde w(s)) \ts^2_d(s)} \\
\label{eqh-7} &\s{\leq} \sqrt{{8|\KI|L_1 \over m(1 - \gamma)^{2d+2}} +
{8|\KI|L_1 \over m} \Delta_{2d+2}},
}
where \eqr{eqh-5} is as in the naive bound. \eqr{eqh-6} is substituting $w(s)$ for $\tilde w(s)$ and \eqr{eqh-6} uses the definition of $\Delta$.
Therefore
\eq{
\Delta_d &\s{\leq} {\epsilon \over 4(1 - \gamma)^d} + B_d + C_d +
\sqrt{{8L_1|\KI| \over m}\left[{1 \over (1 - \gamma)^{2d+2}}\right] } + \sqrt{{8|\KI|L_1 \over m}\Delta_{2d+2}}.
}
Expanding the recurrence up to $\beta$ leads to
\eqn{
\nonumber \Delta_0 &\s{\leq} 8\sum_{d \in \mathcal D - \left\{\beta\right\}}
\left({L_1|\KI| \over m}\right)^{d/(d+2)} \left[{\epsilon \over 4(1 - \gamma)^d} + B_d + C_d + \sqrt{{L_1|\KI| \over m}\left[{1  \over (1 - \gamma)^{2d+2}}\right]} \right]^{2/(d+2)} \\
\label{eqn-hard} &\quad +
8 \left({L_1|\KI| \over m}\right)^{\beta/(\beta + 2)} \left[2\sqrt{{L_1|\KI| \over m(1 - \gamma)^{2\beta+3}}} + B_\beta + C_\beta \right]^{2/(\beta+2)},
}
where we used the naive bound to control $A_\beta$. The bounds on $B_d$ and $C_d$ are somewhat easier, and
follow similar lines to the naive bound on $A_d$.
\eq{
B_d &\s{\equiv} \sum_{s \in X} w(s) {4L_1 \over 3n(s)(1 - \gamma)^{d+1}}
\s{=} {4L_1 \over 3(1 - \gamma)^{d+1}} \sum_{\ki \in \K} {|K(\ki) \over m \kappa}
\s{\leq} {4|\KI|L_1 \over 3m (1 - \gamma)^{d+1}} \\
C_d &\s{\equiv} 2\sum_{s \in X} w(s) \left({L_1 \over n(s)}\right)^{3\over 4} {1 \over (1 - \gamma)^{d+1}}
\s{\leq} {2 \over (1 - \gamma)^{d+1 + 1/4}}\left({|\KI|L_1 \over m}\right)^{3 \over 4}.
}
Letting $m := \constm$ completes the proof.
\end{proofof}

\newpage
\section{Constants}\label{app_const}
The proof of Theorem \ref{thm_upper} uses many constants, which can be hard to keep track of. For convenience we list them below, including
approximate upper/lower bounds as appropriate.
\setlength{\LTleft}{-0.2cm}
\setlength\LTright\fill
\begin{longtable}{l l}
{\bf Constant} & {\bf O/$\Omega$} \\[0.5cm]
$\iotamax := \constiotamax$ & $\O{\log {|S| \over \epsilon(1 - \gamma)}}$ \\[0.5cm]
$\beta := \constbeta$ & $\Omega\left(\log {1 \over 1 - \gamma}\right)$ \\[0.5cm]
$|\mathcal D| := |\mathcal Z(\beta)| $ & $\O{\log \log {1 \over 1 - \gamma}}$ \\[0.5cm]
$|\K| := |\mathcal Z(|S|)|$ & $\O{\log |S|}$ \\[0.5cm]
$|\I| := \iotamax + 1$ & $\O{\log {|S| \over \epsilon(1 - \gamma)}}$ \\[0.5cm]
$|\KI| := |\K||\I|$ & $\O{\log |S| \log {|S| \over \epsilon(1 - \gamma)}}$ \\[0.5cm]
$H := \constH$ & $\O{{1 \over 1 - \gamma} \log{ |S| \over \epsilon(1 - \gamma)}}$ \\[0.5cm]
$\wmin := \constwmin$ & $\Omega\left(\epsilon(1 - \gamma) \over |S|\right)$ \\[0.5cm]
$\delta_1 := \constdeltaone$ & $\O{\delta \over |\SA|^2 \log|S| \log{|S| \over \epsilon(1 - \gamma)}}$ \\[0.5cm]
$L_1 := \log{2 \over \delta_1}$ & $\O{\log{|\SA| \over \delta \epsilon(1 - \gamma)}}$ \\[0.5cm]
$m := \constm$ & $\O{{1 \over \epsilon^2(1 - \gamma)^2 }
\log{|\SA| \over \delta\epsilon(1 - \gamma)} \log{|S|} \log {|S| \over \epsilon(1 - \gamma)} \log^2 \log{1 \over 1 - \gamma}}$ \\[0.5cm]
$N := \constN$ & $\O{{|\SA| \over \epsilon^2(1 - \gamma)^2}
\log{|\SA| \over \delta\epsilon(1 - \gamma)} \log{|S|} \log {|S| \over \epsilon(1 - \gamma)}\log^2 \log{1 \over 1 - \gamma}}$ \\[0.5cm]
$\expmax := \constexpmax$ & $\O{{|\SA| \over \epsilon^2(1 - \gamma)^2}
\log{|\SA| \over \delta\epsilon(1 - \gamma)} \log^2 {|S|} \log^2 {|S| \over \epsilon(1 - \gamma)} \log^2 \log{1 \over 1 - \gamma}}$ \\[0.5cm]
$\umax := \constumax$ & $\O{|\SA| \log|S| \log {|S| \over \epsilon(1 - \gamma)}}$ \\[0.5cm]
\end{longtable}

\newpage
\section{Table of Notation}\label{app_not}

\begin{longtable}{p{2cm} p{11cm}}
$S, A$ & Finite sets of states and actions respectively. \\[0.2cm]
$\gamma$ & The discount fact. Satisfies $\gamma \in (0, 1)$. \\[0.2cm]
$\epsilon$ & The required accuracy. \\[0.2cm]
$\delta$ & The probability that an algorithm makes more mistakes than its sample-complexity. \\[0.2cm]
$\N$ & The natural numbers, starting at $0$. \\[0.2cm]
$\log$ & The natural logarithm. \\[0.2cm]
$\wedge, \vee$ & Logical and/or respectively. \\[0.2cm]
$\E X, \Var X$ & The expectation and variance of random variable $X$ respectively. \\[0.2cm]
$z_i$ & $z_i := 2^i - 2$. \\[0.2cm]
$\mathcal Z(a)$ & Defined as a set of all $z_i$ up to and including $a$.
Formally $\mathcal Z(a) := \left\{ z_i : i \leq \argmin_{i} \left\{z_i \geq a\right\}\right\}$. Contains approximately $\log a$ elements. \\[0.2cm]
$\pi$ & A policy. \\[0.2cm]
$p$ & The transition function, $p: S\times A \times S\to[0,1]$.
We also write $p_{s,a}^{s'} := p(s, a, s')$ for the probability of transitioning to state
$s'$ from state $s$ when taking action $a$. $p_{s,\pi}^{s'} := p_{s,\pi(s)}^{s'}$.
$p_{s,a} \in [0,1]^{|S|}$ is the vector of transition probabilities. \\[0.2cm]
$\hat p, \tilde p$ & Other transition probabilities, as above. \\[0.2cm]
$r$ & The reward function $r:S \to A$. \\[0.2cm]
$M$ & The true MDP. $M := (S, A, p, r, \gamma)$. \\[0.2cm]
$\hM$ & The MDP with empirically estimated transition probabilities. $\hM := (S, A, \hp, r, \gamma)$. \\[0.2cm]
$\tM$ & An MDP in the model class, $\mathcal M$. $\tM := (S, A, \tp, r, \gamma)$. \\[0.2cm]
$V^\pi_M$ & The value function for policy $\pi$ in MDP $M$. Can either be viewed as a function $V^\pi_M : S\to \R$ or vector $V^\pi_M \in \R^{|S|}$. \\[0.2cm]
$\tV^\pi, \hV^\pi$ & The values of policy $\pi$ in MDPs $\tM$ and $\hM$ respectively. \\[0.2cm]
$\pi^* \equiv \pi^*_M$ & The optimal policy in MDP $M$. \\[0.2cm]
$\tilde \pi^* \equiv \pi^*_{\tM}$ & The optimal policy in $\tM$. \\[0.2cm]
$\hat \pi^* \equiv \pi^*_{\hM}$ & The optimal policy in $\hM$. \\[0.2cm]
$\pi_k$ & The (stationary) policy at used in episode $k$. \\[0.2cm]
$n_t(s, a)$ & The number of visits to state/action pair $(s, a)$ at time-step $t$. \\[0.2cm]
$n_t(s, a, s')$ & The number of visits to state $s'$ from state $s$ when taking action $a$ at time-step. \\[0.2cm]
$n_t(s)$ & The number of visits to state/action pair $(s, \pi_t(s))$ at time-step $t$. \\[0.2cm]
$v_{t_k}(s,a)$ & If $t_k$ is the start of an exploration phase then this is the total number of visits to state $(s, a)$ in that
exploration phase. \\[0.2cm]
$s_t, a_t$ & The state and action in time-step $t$ respectively. \\[0.2cm]
$V^\pi_d$ & A higher ``moment'' value function. See Definition \ref{def_recurrence}. \\[0.2cm]
$\sigma^\pi_d(s)^2$ & The variance of $V_d(s')$ when taking action $\pi(s)$ in state $s'$. Defined in Definition \ref{def_recurrence}. \\[0.2cm]
$L_1$ & Defined as $\log(2 / \delta_1)$. \\[0.2cm]
$\mathcal D$ & Defined as $\mathcal Z(\beta)$.  \\[0.2cm]
$w_t(s)$ & The expected discounted number of visits to state $s, \pi_k(s)$ while following policy $\pi_k$. \\[0.2cm]
$X_t$ & The active set containing states $s$ where $w(s) \geq \wmin$. \\[0.2cm]
$\K$ & A set if indices, $\K := \mathcal Z(|S|)$. \\[0.2cm]
$\I$ & A set of indices, $\I := \left\{0,1,2,\cdots, \iotamax\right\}$. \\[0.2cm]
$K_t(\ki)$ & A set of states that have
\eq{
w_t(s) \in [w_\iota, 2w_\iota) \wedge n_t(s) \in m[\kappa w_{\iota}, (2\kappa + 2) w_{\iota}).
}
Note that $\bigcup_{\ki} K_t(\ki)$ contains all states with $w(s) \geq \wmin$. \\[0.2cm]
\end{longtable}

\end{document}